\documentclass[letterpaper]{article} 
\usepackage{aaai24}  
\usepackage{times}  
\usepackage{helvet}  
\usepackage{courier}  
\usepackage[hyphens]{url}  
\usepackage{graphicx} 
\usepackage{natbib}  
\usepackage{caption} 
\usepackage{algorithm}

\usepackage{newfloat}
\usepackage{listings}
\DeclareCaptionStyle{ruled}{labelfont=normalfont,labelsep=colon,strut=off} 
\floatstyle{ruled}
\newfloat{listing}{tb}{lst}{}
\floatname{listing}{Listing}
\usepackage[switch, modulo]{lineno}

\usepackage[utf8]{inputenc} 
\usepackage[T1]{fontenc}    
\usepackage{url}            
\usepackage{booktabs}       
\usepackage{amsfonts}       
\usepackage{nicefrac}       
\usepackage{microtype}      
\usepackage{xcolor}         
\usepackage{acronym}
\usepackage{xspace}
\usepackage{listings}
\usepackage{fancyvrb}
\usepackage{algorithm}
\usepackage[noend]{algpseudocode}
\usepackage{newfloat}
\usepackage{enumitem}
\usepackage{paralist}
\usepackage{graphicx}
\usepackage{amsmath}
\usepackage{amsthm}
\usepackage{amssymb}
\usepackage[english]{babel}
\usepackage{caption}
\usepackage{subcaption}

\acrodef{SAM}{Safe Action Model Learning}
\acrodef{Conditional-SAM}{SAM Learning of Conditional Effects}
\acrodef{UQV}{universally quantified variables}

\newtheorem{definition}{Definition}

\newtheorem{theorem}{Theorem}[section]

\newtheorem{example}{Example}
\newcommand{\sam}{\ac{SAM}\xspace}
\newcommand{\consam}{\ac{Conditional-SAM}\xspace}
\newcommand{\tuple}[1]{\ensuremath{\left \langle #1 \right \rangle }}
\newcommand{\true}{\textit{true}\xspace}
\newcommand{\pre}{\textit{pre}}
\newcommand{\eff}{\textit{eff}\xspace}
\newcommand{\name}{\textit{name}}

\newcommand{\realm}{\ensuremath{M^*}\xspace}
\newcommand{\uqv}{\ac{UQV}\xspace}
\newcommand{\uqvs}{UQVs\xspace}

\newcommand{\lifted}{\textit{lifted}}

\newcommand{\pb}{\textit{pb}}

\newcommand{\bindings}{\textit{bindings}\xspace}

\newcommand{\grounding}{\textit{g}\xspace}

\newcommand{\LearnConditions}{BuildActionModel}
\newcommand{\mustber}{\textit{MustBeResult}}
\newcommand{\possiblya}{\textit{PosAnte}}

\newcommand{\mustbea}{\textit{Ante}\xspace}
\newcommand{\mustnotbea}{\textit{NotAnte}\xspace}

\newcommand{\roni}[1]{}
\newcommand{\argaman}[1]{}
\newcommand{\brendan}[1]{}
\newcommand{\enrico}[1]{}

\title{Safe Learning of PDDL Domains with Conditional Effects - Extended Version}

\author{
    Argaman Mordoch$^1$,
    Enrico Scala$^2$,
    Roni Stern$^1$, 
    Brendan Juba$^3$
}
\affiliations{
    \textsuperscript{\rm 1}Ben Gurion University of the Negev\\
    \textsuperscript{\rm 2}Università degli Studi di Brescia\\
    \textsuperscript{\rm 3}Washington University in St. Louis\\

    mordocha@post.bgu.ac.il, enrico.scala@unibs.it, sternron@bgu.ac.il, bjuba@wustl.edu.
}

\usepackage{bibentry}

\begin{document}

\maketitle

\begin{abstract}
Powerful domain-independent planners have been developed to solve various types of planning problems. 
These planners often require a model of the acting agent's actions, given in some planning domain description language. 
Manually designing such an action model is a notoriously challenging task. 
An alternative is to automatically learn action models from observation. 
Such an action model is called safe if every plan created with it is consistent with the real, unknown action model. 
Algorithms for learning such safe action models exist, yet they cannot handle domains with conditional or universal effects, which are common constructs in many planning problems.   
We prove that learning non-trivial safe action models with conditional effects may require an exponential number of samples.
Then, we identify reasonable assumptions under which such learning is tractable and propose \consam\, the first algorithm capable of doing so. 
We analyze \consam\ theoretically and evaluate it experimentally. 
Our results show that the action models learned by \consam can be used to solve perfectly most of the test set problems in most of the experimented domains.
\end{abstract}

\section{Introduction}
Planning is the fundamental task of choosing which actions to perform to achieve a desired outcome. 
An \emph{automated domain-independent planner} refers to an Artificial Intelligence (AI) algorithm capable of solving a wide range of planning problems~\citep{ghallab2016automated}. 
Developing a domain-independent planner is a long-term goal of AI research. 
Researchers developed many domain-independent planners for various types of planning problems. Such planners include Fast Downward~\citep{helmert2006fast}, Fast Forward~\citep{hoffmann2001ff}, ENHSP~\citep{scala2016interval}, and more. 
These planners require a model of the acting agent's actions, given in some domain description language such as the Planning Domain Definition Language (\textbf{PDDL})~\citep{aeronautiques1998pddl}.  
Defining an agent's \textit{action model} to solve real-world problems is extremely hard. 
Researchers acknowledged this modeling challenge, and algorithms for learning action models from observations have been proposed~\citep{cresswell2011generalised,aineto2019learning,yang2007learning,juba2021safe,mordoch2023learning}.

Since the learned model may differ from the domain’s actual action model, it is important to study whether the plans provided offer \textit{execution-soundness} guarantees. In general, the learned model may be too \emph{permissive}, in the sense that it 
allows plans that either cannot be applied in the domain or do not reach a state that satisfies the problem goals. 
We aim for consistency between the validity of a plan as determined by the learned model and its validity within the actual model. This ensures that an agent can confidently execute actions, assured of achieving the goal despite possessing incomplete knowledge of the complete action models.
In problem settings where execution failures are unacceptable or are very costly, e.g., autonomous vehicles, high-end robotics, and medical treatment planning, soundness becomes a hard constraint. 
We focus on such cases and aim to learn an action model that satisfies the strongest form of soundness: every plan generated using the learned model must be applicable and yield the same states in the real, unknown model. An action model that satisfies this requirement has been called \emph{safe}~\cite{juba2021safe,juba2022learning,mordoch2023learning}.\footnote{This notion of safety has different definitions in different contexts. For example, safe Reinforcement Learning often refers to ensuring that some safety function of the current state never goes below some threshold value~\cite{wachi2020safe}.}
We view this as a ``safety'' notion in part since it enables more conventional notions of safety to be enforced during planning, and provides assurance that they will carry over to the actual execution.


Algorithms from the \sam{} family~\citep{stern2017efficientAndSafe,juba2021safe,juba2022learning,mordoch2023learning} address the challenge of learning safe action models under different sets of assumptions. 
However, these algorithms are not suitable for learning actions that may include \textit{conditional effects}. 
A conditional effect is an effect that occurs only when a specific condition holds. 
For example, consider an AI for planning treatments for patients and the action of giving a flu medicine to a patient, where that medicine causes an allergic reaction in patients with a particular rare blood type. This action's effects include not having the flu, but there is also a conditional effect specifying that an allergic reaction occurs if the patient has a rare blood type. 
If the patient has a rare blood type, we would want to avoid applying this type of treatment. A safe action model would never permit the execution of the action in these cases.


Previous works on action model learning with conditional effects~\citep{oates1996learning,zhuo2010learning} made no safety guarantees for the learned models.
This work addresses this gap by exploring the problem of learning safe action models for PDDL~\citep{aeronautiques1998pddl} domains with conditional effects. 
Specifically, we introduce the \consam{} algorithm, which is guaranteed to output a safe action model.
We show that \consam requires an asymptotically optimal number of trajectories when the size of the antecedents for the conditional effects is restricted, which is the only case where the problem is tractable. 
Then, we describe how \consam\ can be extended to support \emph{lifted action models} (i.e., parameterized) and effects with \textit{universally quantified variables}.
Finally, we demonstrate the usefulness of \consam{} in practice on benchmark planning problems with conditional and universal effects. Our results show that given a few observations, the model \consam learns is logically identical to the real action models and can be used to solve test problems for most of the experimented domains.

\section{Preliminaries}

\label{sec:planning-problem}
We focus on planning problems in domains where action outcomes are deterministic, the states are fully observable and contain Boolean variables only.
Such problems are commonly modeled using a fragment of the ADL (Action Description Language)~\citep{pednault1989adl} and formulated in PDDL (Planning Domain Definition Language)~\citep{aeronautiques1998pddl}. 
In PDDL, a planning problem is described by a PDDL domain and a PDDL problem.  
A PDDL domain is a tuple $D=\tuple{F,A}$ 
where $F$ is a finite set of Boolean variables, referred to as fluents 
and $A$ is a set of actions. 
A \emph{literal} refers to either a fluent or its negation. Let $L$ be the set of every possible literal. 
Note that $|L|=2|F|$. 
A \emph{state} is a conjunction of literals that includes, for every fluent $f$, either $f$ or $\neg f$. 
The value of a fluent $f$ is a state $s$ is true if $s$ includes $f$ and false otherwise. 
An action $a\in A$ is a triple $\tuple{\name(a), \pre(a), \eff(a)}$ 
corresponding to the name of the action, its preconditions, and effects. 
The preconditions of an action $a$, $\pre(a)$, is a conjunction of literals that are sufficient and necessary conditions for applying $a$. 
If the preconditions of $a$ are satisfied in a state $s$ we say that $a$ is applicable in $s$. 
If an action has no preconditions, then it is applicable in any state. 
The effects of an action $a$, $\eff(a)$, specify the outcome of applying $a$. 
An effect is defined by a tuple $\tuple{c,e}$ where $c$ is called the antecedent (condition) and $e$ is called the result. 
Both $c$ and $e$ are conjunctions of literals. 
The semantics of an effect $\tuple{c,e}$ for an action $a$ is 
that if $a$ is applied in a state $s$ where the antecedent $c$ holds then the result $e$ will be true in the next state. 
The antecedent can also be $\true$, representing that the result occurs regardless of the state where the action has been applied. 
An effect where the antecedent is not $\true$ is called a \emph{conditional effect}. 
The outcome of applying $a$ to a state $s$, denoted $a(s)$, is a state in which the value of every fluent is as in state $s$ except those fluents changed by the action's effects.

A PDDL \emph{problem} is defined by a tuple $P=\tuple{I, G, D}$ where $I$ is the initial state of the world, 
$G$ is a conjunction of literals that define the desired goal, and $D$ is a PDDL domain. 
A \textit{plan} $\Pi=\tuple{a_1,a_2,...a_n}$ is a sequence of actions. 
A plan $\Pi$ is called \textit{valid} for a PDDL problem $P=\tuple{I, G, D}$ if $a_1$ is applicable in $I$, $a_i$ is applicable in $a_{i-1}(\cdots(a_1(I))\cdots)$, and $G\subseteq a_n(a_{n-1}(\cdots(a_1(I))\cdots))$.


An \emph{action model} for a PDDL domain $D=\tuple{F,A}$ is a pair $M=\tuple{\pre_M, \eff_M}$ 
where $\pre_M$ maps every action in $A$ to a (possibly empty) conjunction of literals in $F$, 
and $\eff_M$ maps actions in $A$ to a (possibly empty) set of effects over $F$. 
The \emph{real action model} of a domain, denoted $M^*$, is the action model where every action $a$ is mapped to its real precondition and effects, i.e., $\pre_{M^*}(a)=\pre(a)$ and $\eff_{M^*}(a)=\eff(a)$. 
For an action $a$, state $s$, and action model $M$, we denote by $a_M(s)$ 
the state that results from applying $a$ in $s$ assuming that $M$ is the real action model. 

\begin{definition}[Safe Action Model]
\label{def:safe_action_model}
An action model $M$ is safe w.r.t. an action model $M'$ if for every state $s$ and action $a$ it holds that 
if $a$ is applicable in $s$ according to $M$ then 
(1) it is also applicable in $s$ according to $M'$, and
(2) applying $a$ in $s$ results in exactly the same state according to both $M$ and $M'$. 
Formally: 
\begin{equation}
\small
    \pre_{M}(a)\subseteq s\rightarrow \left(\pre_M'(a)\subseteq s 
            \wedge a_{M}(s)=a_M'(s)\right)
        \label{eq:safety}
\end{equation}
\end{definition}
An action model is said to be safe in a domain if it is safe w.r.t. its real action model. 
This paper deals with the case where the planning agent does not know the real action model of a given domain, yet it aims to learn an action model that is safe in it. 
A major benefit of learning such a safe action model is that any plan generated with the learned action model for any problem in the same domain is also valid with respect to the real, unknown, action model. 
Following prior works on learning action models~\cite{amir2008learning,cresswell2013acquiring,aineto2018learning} in general and safe action models in particular~\cite{stern2017efficientAndSafe,juba2021safe,mordoch2023learning}, 
we assume as input a set of observations of previously executed plans, represented as a set of \textit{trajectories}. 
A trajectory $T=\tuple{s_0, a_1, s_1, \ldots a_n, s_n}$ is an alternating sequence of states $(s_0,\ldots,s_n)$ and actions $(a_1,\ldots,a_n)$ that starts and ends with a state. 
The trajectory created by applying $\pi$ to a state $s$ is 
the sequence $\tuple{s_0, a_1, \ldots, a_{|\pi|}, s_{|\pi|}}$ such that 
$s_0=s$ and for all $0<i\leq |\pi|$, $s_i=a_i(s_{i-1})$. 
A trajectory is often represented as a set of \emph{action triplets}
$\big\{\tuple{s_{i-1},a_i,s_i}\big\}_{i=1}^{|\pi|}$.

\subsection{Problem Definition and Assumptions}
We deal with the problem of learning a safe action model for a domain $D$ given a set of trajectories $\mathcal{T}$ collected by executing plans for different problems in $D$. 
Ideally, the learned action model will be able to generalize beyond the given set of trajectories and enable finding plans for other problems in $D$.


\noindent We make the following assumptions:
\begin{compactenum}
    \item \label{ref:assumption1} The given trajectories are fully observable and noise-free. 
    \item\label{ref:assumption2} For each literal $l'$ and action $a$, there is \textit{at most} one effect of $a$ for which $l'$ is a result. 
    \item \label{ref:assumption3} The maximal number of literals in an antecedent is at most $n$, a fixed parameter known in advance.
\end{compactenum}

Assumption~\ref{ref:assumption1} means we observe all the actions and the values of all the fluents in all states in every trajectory in $\mathcal{O}$. This assumption is common in the action model learning literature and relaxing it in the context of conditional effects is left for future work. 
Assumption~\ref{ref:assumption2} means that there are no \emph{disjunctive antecedents}, i.e., multiple effects for the same action having the same result but with different antecedents. 
The implication of this assumption is that if $(c, e)$ is an effect of some action $a$, then no conjunction of literals except $c$ is an antecedent of the literals in the result $e$. Formally: 
    \begin{multline}
    \forall l',a: ( (c,e)\in\eff(a): l'\in e ) \\
    \rightarrow (\nexists (c',e')\in\eff(a): l'\in e' \wedge c\neq c')
    \end{multline}
This assumption crucially improves the efficiency of learning. 
We discuss relaxing this assumption later. 
Assumption~\ref{ref:assumption3} means that a human modeler must specify an upper bound on the number of literals in an antecedent for the domain at hand. Specifying such a bound is significantly easier than manually defining the entire action model. 
We prove later that without the third assumption, learning conditional effects is intractable. 


\section{Approach}
\consam learns an action model by applying the following rules:

\begin{definition}[\consam Inductive Rules]\label{def:cond-sam-rules}\label{def:basic-rules}
For every action triplet $\tuple{s,a,s'}\in \mathcal{T}$ 
\begin{compactenum}
    \item\label{rule:not-a-precondition} [Not a precondition] For every literal $l\notin s$, $l\notin \pre(a)$ 
    \item\label{rule:not-an-effect} [Not a result] For every literal $l'\notin s'$, 
    $\nexists(c,e)\in\eff(a)$ where 
    $(c\wedge s\nvdash\bot) \wedge (l'\in e)$  
    \item\label{rule:must-be-effect} [Must be an effect]  For every literal $l' \in s'\setminus s$, $\exists (c,e)\in \eff(a): 
    (c\wedge s\nvdash\bot) \wedge 
    (l'\in e)$ 
    \item\label{rule:no-disjunctive} [Not an antecedent] For every literal $l'\in s' \setminus s$, and conjunction of literals $c$: if $c\wedge s\vdash\bot$, then $\nexists(c',e)\in \eff(a)$ such that $c \subseteq c'$ and $l'\in e$. 
\end{compactenum}
\end{definition}

\noindent The first three rules generalize the \sam-Learning~\citep{juba2021safe} inductive rules to support conditional effects. 
Rule 4 is derived from Assumption~\ref{ref:assumption2} (no conditional effects with disjunctive antecedents): if we observe $l'$ as the result of the action, any conjunction of literals $c$ that is not satisfied in $s$ cannot be the antecedent of the conditional effect for $l'$.

\begin{example}
Consider a domain with three fluents $f_1$, $f_2$, and $f_3$, where the size of antecedents is bounded by 1 (i.e., $n=1$), and assume a Boolean vector of size 3 represents a state. 
Now, assume we observed an action triplet $\tuple{(T,T,F),a,(F,T,F)}$. 
Using the Rule 1 in Def.~\ref{def:cond-sam-rules}, we infer that $\neg f_1$, $\neg f_2$, and $f_3$ are not preconditions of the action $a$.
By applying Rule 2 in Def.~\ref{def:cond-sam-rules}, $a$ cannot include an effect $(c,e)$ such that $c$ is consistent with $f_1\wedge f_2\wedge \neg f_3$ and the result is either $f_1$, $\neg f_2$, or $f_3$. 
Since $n=1$, this rules out the conditional effects where $c$ is one of the following $\{\true, f_1, f_2, \neg f_3\}$ and $e$ is either $f_1$ or $ \neg f_2$ or $f_3$, e.g., $(c,e)  = (f_1, \neg f_2)$. 
According to 
Rule 3 in Def.~\ref{def:cond-sam-rules} there exists $(c,e)\in\eff(a)$ such that $c$ is one of $\{\true, f_1, f_2, \neg f_3\}$ and $e=\neg f_1$.
Finally, according to 
Rule 4 in Def.~\ref{def:cond-sam-rules} $(\neg f_2,\neg f_1)$ and $(f_3, \neg f_1)$ cannot be conditional effects. 
\end{example}

\subsection{Conditional-SAM Algorithm}
Next, we describe the \consam{} algorithm, which uses the \consam{} inductive learning rules (Definition~\ref{def:basic-rules}). 
The pseudo-code for \consam{} is given in Algorithm~\ref{alg:consam-algorithm}. 
Let $A(\mathcal{T}), L(\mathcal{T})$ be the set of actions and literals observed in the trajectories $\mathcal{T}$. 
\consam{} maintains three data structures: \pre$(a)$ and \mustber$(a)$ for every action $a$, and \possiblya$(l, a)$ for every action $a$ and literal $l$. 
\pre$(a)$ is a set of literals, representing which literals may be preconditions of $a$. 
It is initialized to all the literals $l\in L(\mathcal{T})$ (line~\ref{alg:init-pre}).
\possiblya$(l, a)$ is a set of conjunctions of literals, representing all the conjunctions that may be antecedents of a conditional effect of $a$ that results in $l$.\footnote{According to Assumption~\ref{ref:assumption2}, in the real action model there can be only one such conjunction. \consam{} maintains in \possiblya$(l,a)$ a set of conjunctions since it does not know the real action model.} 
This data structure is initialized to include every conjunction of literals of size $n$ or less (line~\ref{alg:init-posante}).  
\mustber$(a)$ maintains the set of literals observed to be a result of applying $a$. 
This data structure is initialized as an empty set (line~\ref{alg:init-mustberes}). 
\consam{} updates these data structures by applying the inductive learning rules for each action triplet in the given trajectories. 
That is, it removes literals from \pre$(\cdot)$ according to Rule~\ref{rule:not-a-precondition}, 
removes conjunctions of literals from \possiblya$(l, a)$ using Rules~\ref{rule:not-an-effect} and~\ref{rule:no-disjunctive}, 
and adds literals to \mustber$(a)$ using Rule~\ref{rule:must-be-effect} (lines~\ref{alg:applying-rules}-~\ref{alg:applying-rules-end}). 





Then, \consam{} iterates over every action $a\in A(\mathcal{T})$  
using $\pre(a)$, $\possiblya(\cdot, a)$, and $\mustber(a)$ to generate 
the preconditions and effects of $a$ in the resulting safe action model. 
This part of \consam{} is encapsulated in the function \textit{\LearnConditions}, listed in Algorithm~\ref{alg:learn-conditionals}. 
\LearnConditions{} 
stores the preconditions and effects of the resulting safe action model in 
$\pre^*(a)$ and $\eff^*(a)$, respectively. 
Initially, $\eff^*(a)$ is an empty set and $\pre^*(a)$ is set to be $\pre(a)$. 
Then, it iterates over every literal $l$ and considers adding an effect to $\eff^*(a)$ with a $l$ as a result, as follows. 
Let $PA$ be the subset of $\possiblya(l,a)$ containing only conjunctions that are disjoint from $\pre(a)$ (line~\ref{alg:remove-preconditions}). 
\consam{} uses $PA$ to compute two formulas, $\mustbea$ and $\mustnotbea$. 
$\mustnotbea$ is the conjunction of the negation of every clause $c$ in $PA$, 
and $\mustbea$ is the conjunction of all the clauses $c\in PA$ (lines~\ref{alg:mustnotber-init},~\ref{alg:mustber-init}). 
Observe that applying $a$ in a state where $\mustbea$ is true guarantees that $l$ will be true in the subsequent state. 
Similarly, applying $a$ in a state where $\mustnotbea$ is true guarantees that $l$ will not be true in the subsequent state unless it was true before. 
We apply unit propagation on the clauses to minimize their number size (line~\ref{alg:unit-prop}). 
Afterward, the function verifies whether $l\in\mustber(a)$. If so, the tuple $(\mustbea, l)$ is added to $\eff(a)$. 
If $\possiblya(l,a)$ includes more than a single clause of possible antecedents, then there is an ambiguity on which antecedent causes $l$. To mitigate this, \consam{} adds to $\pre^*(a)$ the disjunction $(l\vee\mustnotbea\vee\mustbea)$ (line~\ref{alg:add-all-possible-pre}). 
This disjunction is composed of three parts as follows:
First, allowing the action to be applicable if the result, $l$, is observed in the pre-state. 
Second, the action is permitted if \emph{none} of the antecedents hold in the pre-state.
Last, $a$ is applicable if \emph{all} the antecedents hold in the pre-state. 
If one of the above holds, the action can be executed. 

If $l$ was not observed as a result of the action, i.e., $l\notin\mustber(a)$, the function adds $(l\vee\mustnotbea)$ to $\pre^*(a)$ (line~\ref{alg:add-not-effect-restrictive}).

Since we have yet to observe $l$ as a result of the action, then $l\notin\eff(a)$; 
Thus, to prevent $l$ from triggering unexpectedly, we do not permit the action to be executed if $\mustbea$ is true.
After repeating this for every action, \textit{\LearnConditions} returns the safe action model comprising $\pre^*$ and $\eff^*$.

\begin{algorithm}[ht]
\small
\caption{\consam Algorithm}\label{alg:consam-algorithm}
\begin{algorithmic}[1]
\State \textbf{Input}: $\mathcal{T}, n$ 
\State \textbf{Output}: A safe action model.
\For{$a \in A(\mathcal{T})$}
    \State $\pre(a) \gets L(\mathcal{T})$\label{alg:init-pre} 
    \State $\mustber(a) \gets \emptyset$ \label{alg:init-mustberes}  
    \State $\possiblya(l,a) \gets \bigcup_{i=1}^n \{l_1\wedge ...\wedge l_i | \forall 1\leq j\leq i: l_j\in L(\mathcal{T}) \} \cup \{\true\}$\label{alg:init-posante} 
\EndFor\label{alg:init-end}
\For{$\tuple{s,a,s'}\in \mathcal{T}$} \label{alg:applying-rules} 
    \For{$l$  such that $l \notin s$} 
        \State $pre(a) \gets pre(a) \setminus \{l\}$\label{alg:rule1}  \Comment{Rule \ref{rule:not-a-precondition}}
    \EndFor
    \For{$l\in s'\setminus s$} \Comment{Rule \ref{rule:must-be-effect}}
        \State $\mustber(a)$ $\gets$ $\mustber(a) \cup \{l\}$
    \EndFor    
    \For{$l' \notin s'$ and $c \in \possiblya(l',a)$ s.t.  $(c \wedge s\nvdash\bot)$} 
        \State $\possiblya(l',a) \gets \possiblya(l',a)\setminus c$\label{alg:rule4} 
    \EndFor \Comment{Rule \ref{rule:not-an-effect}}
    \For{$l' \in s'\setminus s$ and $c \in \possiblya(l',a)$ s.t. $c\wedge s\vdash\bot$} \Comment{Rule \ref{rule:no-disjunctive}}
        \State $\possiblya(l',a) \gets \possiblya(l',a)\setminus c$\label{alg:rule5} 
    \EndFor        
\EndFor\label{alg:applying-rules-end}
\State \Return  \textbf{\LearnConditions}$(\pre,\mustber,\possiblya)$
\end{algorithmic}
\end{algorithm}

\begin{algorithm}[ht]
\small
\caption{\LearnConditions{}}\label{alg:learn-conditionals}
\begin{algorithmic}[1]
\State \textbf{Input}: $\pre,\mustber,\possiblya$ 
\State \textbf{Output}: $\pre^*$ and $\eff^*$ for all actions.
\For{$a \in A(\mathcal{T})$}
    \State $\eff(a) \gets \emptyset$; ~~ $\pre^*(a) \gets \bigwedge_{l\in \pre(a)}l$
    \For{$l\in L(\mathcal{T})\setminus\pre(a)$ where $\possiblya(l,a)\neq\emptyset$} 
        \State $PA\gets \{c\in \possiblya(l,a)|(\pre(a)\cap c) = \emptyset\}$\label{alg:remove-preconditions}

        \State $\mustnotbea \gets \bigwedge_{c\in PA}\neg c$\label{alg:mustnotber-init}
        \State $\mustbea \gets \bigwedge_{c\in PA} c$\label{alg:mustber-init}
        \State Minimize $\mustbea$ and $\mustnotbea$ using unit propagation.\label{alg:unit-prop}
        \If{$l \in \mustber(a)$}
            \State Add to $\eff(a)$: $(\mustbea, l)$ 
            \If{$PA$ is not a single clause}
                \State $\pre^*(a) \gets \pre^*(a) \wedge (l\vee \mustnotbea \vee \mustbea)$  \label{alg:add-all-possible-pre}
            \EndIf          
        \Else
            \State $\pre^*(a) \gets \pre^*(a) \wedge (l\vee \mustnotbea)$ \label{alg:add-not-effect-restrictive}
        \EndIf
    \EndFor
\EndFor
\State \Return  $\tuple{\pre^*, \eff}$
\end{algorithmic}
\end{algorithm}

\begin{example}
Given a domain with 3 literals and an action $a$ where $\pre(a)=\emptyset$, 
$l_1\in \mustber(a)$, 
$\possiblya(l_1, a)=\{\{l_2\},\{l_3\}\}$, 
$\possiblya(l_2, a)=\emptyset$, and $\possiblya(l_3, a)=\emptyset$. 
The resulting preconditions and effects after applying \textit{\LearnConditions} are $\pre^*(a)=(l_1)\vee (\neg l_2\wedge \neg l_3)\vee(l_2\wedge l_3)$  and 
$\eff^*(a)=(l_2\wedge l_3, l_1)$, i.e., when $l_2\wedge l_3$ then $l_1$.
\end{example}


\begin{theorem}\label{lem:sam-safe}
    The action model $M'$ learned by \consam is safe w.r.t. the action model that generated the input trajectories $\mathcal{T}$.
\end{theorem}
\begin{proof}
Let \realm be the real action model, i.e., the one used to generate the input trajectories. 
    The preconditions for every action $a$ in $M'$ are a superset of the preconditions for $a$ in $\realm$. 
    Thus, for each action $a$ and state $s$ such that $a$ is applicable according to $M'$, it is guaranteed to be applicable according to \realm.
    Next, we prove that 
    for every action $a$ that is applicable in the state $s$ according to $M'$, then $s' = a_{M'}(s)$, is equivalent to $s'^*=a_\realm(s)$.     
    By contradiction, assume that $s'\neq s'^*$. 
    This means either (1) $\exists l\in s'$ such that $l\notin s'^*$, or (2) $\exists l\notin s'$ such that $l\in s'^*$.
    Since \consam only adds effects observed in the trajectories, according to Rule~\ref{rule:must-be-effect}, there cannot be a literal $l$ such that (1) holds.
    If $l\in s'^*$ but $l\notin s'$, then \consam did not observe $l$ as a result of $a$ and thus did not add it as an effect. 
    According to line~\ref{alg:add-not-effect-restrictive} one of $(l\vee \mustnotbea)$ hold in $s$. 
    If $l\in s$, then $l\in s'$ according to $M'$ (since $a$ does not remove it), which contradicts (2).
    Similarly, if $\mustnotbea\subseteq s$, then the antecedent of $l$ according to \realm is negated in $s$ thus $l\notin s'^*$ which also contradicts (2).
\end{proof}

\section{Theoretical Analysis}

Next, we analyze the \consam algorithm. We prove that under a fixed antecedent size ($n$) its space, runtime, and sample complexity are tractable, and show that our sample complexity bound is tight. 

\begin{theorem}  
\label{lemma:space}
The space complexity of \consam is $O\left(|A||F|^{n+1}\left(\frac{e}{n}\right)^n\right)$, 
where $e$ is the base of the natural logarithm. 
\end{theorem}

\begin{proof}
    For every action $a\in A$ and every literal $l\in L$, \consam maintains the data structures $\pre(a)$, $\possiblya(l,a)$ and $\mustber(a)$.
    The size of $\pre(a)$ is at most $|L|$. 
    The size of $\mustber(a)$ is also at most $|L|$. 
    The size of $\possiblya(l, a)$ is observed when it is initialized, containing every conjunction of literals of size at most $n$, including the empty set (representing the antecedent $\true$). 
    Thus, the size of $\possiblya(l,a)$ is at most $\sum_{i=0}^n{|L|\choose i}\leq \left(\frac{|L|\cdot e}{n}\right)^n$.
    We note that the space complexity of \textit{\LearnConditions} is linear in the size of $\possiblya$.
    Consequentially, the space complexity is
    \begin{align*}\small
        |A||L|+|A||L|+|A||L|\sum_{i=0}^n{|L|\choose i} = \\
        |A|\cdot2|F|+|A|\cdot2|F|+|A|\cdot2|F|\sum_{i=0}^n{2|F|\choose i} \\
         \in O\left(|A||F|^{n+1}\left(\frac{e}{n}\right)^n\right)
    \end{align*}
    Recall that $n$ is a fixed constant --- the maximal number of literals in an antecedent.     
\end{proof}

\begin{theorem}
    \label{lemma:runtime}
    The runtime complexity of \consam is 
        $O\left(|A||F|^{n}\left(\frac{e}{n})^n\right)+|\mathcal{T}||F|^{n+1}\left(\frac{e}{n})^n\right)\right)$
\end{theorem}

\begin{proof}
    The initialization process requires the same runtime complexity as its space complexity, i.e., $O(|A||F|^{n+1}\left(\frac{e}{n}\right)^n)$.
    Then, \consam iterates over all action triplets and applies the inductive rules in Def. 2. 
    This requires $O(|\mathcal{T}||F|^{n+1}\left(\frac{e}{n}\right)^n)$, 
    since as discussed above, the size of $\possiblya(l,a)$ is at most $|F|^{n}\left(\frac{e}{n}\right)^n$. 
    
    Finally, in the \textit{\LearnConditions} function, the runtime complexity is bounded by the most intensive computational part, which is the part that creates the restrictive conditions.
    The complexity of this part is linear in $\possiblya(l, a)$.
    Thus the total runtime complexity of \textit{\LearnConditions} is bounded by $O\left(|A||F|^{n}\left(\frac{e}{n})^n\right)\right)$
    Thus, the total runtime complexity of the algorithm is $O\left(|A||F|^{n}\left(\frac{e}{n})^n\right)+|\mathcal{T}||F|^{n+1}\left(\frac{e}{n})^n\right)\right)$.
\end{proof}

As can be seen from Theorems~\ref{lemma:space} and \ref{lemma:runtime}, the complexity of the algorithm is independent of the number of effects and is only affected by the maximal size of the antecedents and the number of literals and actions in the domain. 
The complexity does, however, increase exponentially with $n$. 

Next, we present an analysis of the algorithm's sample complexity as well as a lower bound on the number of samples needed to learn safe and approximately complete models. 

\begin{theorem}
Let $\mathcal{D}$ be a distribution over pairs $\tuple{P,\Pi}$ where $P$ is a problem from a fixed domain $D$ and $\Pi$ is a plan solving $P$. Given 
\begin{equation*}
\small
m\geq \frac{1}{\epsilon}\left(\ln(3)|F||A|+2\ln(2)|F||A|\left(\frac{2|F|e}{n}\right)^n+\ln\frac{1}{\delta}\right)
\end{equation*}
trajectories obtained by executing $\Pi$ for $m$ independent draws from $\mathcal{D}$, \consam returns an action model $M'$ such that with probability $1-\delta$, for a new $P$ drawn from $\mathcal{D}$, the probability that there exists a plan consistent with $M'$ solving $P$ is at least $1-\epsilon$.
\end{theorem}

At a high level, the sample complexity follows since the learned preconditions of each action of a plan $\Pi$ sampled from $\mathcal{D}$ are satisfied and the action model is safe. Either at least one literal is deleted from $\pre$ or at least one clause $c$ is deleted from some $\possiblya(e,a)$ when such an action $a$ would be used by the plan $\Pi$. Thus if a specific literal or clause would prohibit the action with probability greater than $\epsilon$, that literal/clause is eliminated with high probability given a sample of the specified size.

\begin{proof}
In view of Theorem~\ref{lem:sam-safe}, it suffices to show that for a pair $\tuple{P,\Pi}$ drawn from $\mathcal{D}$, the preconditions of $\Pi$ in $M'$ are satisfied for each step of the execution of $\Pi$ in the real action model $M^*$; indeed, the states obtained by $M'$ and $M^*$ are identical, so $\Pi$ will then also solve $P$ in $M'$.

Recall that \consam passes the sets $\pre$, $\mustber$, and $\possiblya$ for each action $a$ and, in the case of $\possiblya$, for each literal $l$ to Algorithm~\ref{alg:learn-conditionals}.
A literal $l$ only appears in $\pre(a)$ for an action $a$ if $\neg l$ has never been observed in the pre-state when action $a$ was taken. Similarly, a clause $\neg c$ may appear as (a subclause of) some clause of the precondition $\pre^*$ of $a$ in $M'$ if $c$ remains in the antecedents set $\possiblya(e,a)$ of some candidate effect $e\in \mustber(a)$ for which more than one such candidate remains, or for which $e\notin \mustber(a)$ and $c$ is in $\possiblya(e,a)$. Note that if $\neg c$ is falsified in a state $s$ (prohibiting $a$ in $s$ in $M'$), $s\subseteq c$. Hence, if the execution of $\Pi$ in $M^*$ would result in $a$ being taken in $s$ resulting in $s'$, \consam would remove $c$ from $\possiblya(e,a)$ for all $e\notin s'$, and $c$ from $\possiblya(e,a)$ if $c\nsubseteq s$ and $e\in s'\setminus s$. 

We now claim that the probability that \consam obtains a set of antecedents $\possiblya(l,a)$ and set of preconditions $\pre(a)$ that prohibits the execution of $\Pi$ with probability greater than $\epsilon$ is at most $\delta$: $a$ is only prohibited by $\pre^*$ in $s$ if (1) $l\in s$ for some $\neg l\in pre(a)$; if (2) $s\subseteq c$ for some $c\in \possiblya(l,a)$ where $l\notin \mustber(a)$ and $\neg l\in s$; or, if (3) $l\in \mustber(a)$, $\neg l\in s$,  $s\subseteq c$ for some $c\in \possiblya(l,a)$, and $s\nsubseteq c'$ for some (other) $c'\in \possiblya(l,a)$. When the execution of $\Pi$ includes taking such an action $a$ in such a prohibited state $s$, in the first case we see \consam removes the falsified $\neg l$ from $\pre(a)$; for every effect $e$ of $a$ in $s$, any $c\nsubseteq s$ are removed from $\possiblya(e,a)$ so cases (2) and (3) cannot occur; and for every literal $\tilde{e}$ that is not an effect of $a$ in $s$, since $\tilde{e}\notin s'$, all $(c,\tilde{e})$ for $c\subseteq s$ are removed from $\possiblya(\tilde{e},a)$, so neither case (2) nor (3) can occur. Thus, we see that either at least one literal is deleted from $\pre$ or at least one $c$ is deleted from some $\possiblya(e,a)$ when such an $(s,a,s')$ occurs in the trajectory, so that $a$ is permitted by $\pre^*(a)$ in $s$ subsequently. Since literals are only deleted from $\pre$ and clauses are only deleted from $\possiblya$, \consam then cannot return the eliminated collection of preconditions and antecedents sets.

Quantitatively, for any collection of preconditions and antecedents sets for which such a problem and plan would be obtained from $\mathcal{D}$ with probability greater than $\epsilon$, \consam can only return the corresponding collection with probability at most $(1-\epsilon)^m$ when it is given $m$ examples drawn independently from $\mathcal{D}$. Observe that there are $3^{|F|}$ possible sets $\pre$ for each $a\in A$, and $2^{\sum_{k=0}^n2^k{|F|\choose k}}$ possible sets $\possiblya$ for each $l$ and $a$.  Thus, there are at most
\[
3^{|F||A|}2^{2|F||A|\sum_{k=0}^n2^k{|F|\choose k}}\leq e^{\ln(3)|F||A|+2\ln(2)|F||A|\left(\frac{2|F|e}{n}\right)^n}
\]
possible collections of $pre$ and $\possiblya$. Since $(1-\epsilon)^m\leq e^{-m\epsilon}$, taking a union bound over all possible collections of $\pre$ and $\possiblya$ that prohibit the execution of the associated plan with probability at least $\epsilon$, we find that for the given $m$, the total probability of \consam obtaining such a collection of preconditions and antecedents sets is at most $\delta$. Thus, with probability $1-\delta$, the action model indeed permits executing the plans associated with problems drawn from $\mathcal{D}$ with probability at least $1-\epsilon$ as needed.
\end{proof}

\consam, therefore, enjoys approximate completeness with high probability so long as the number of training trajectories is sufficiently large. The one unsatisfying aspect of our bound is that the number of trajectories is exponential in the size of the antecedents of the conditions in the conditional effects we consider. Unfortunately, we find that this is unavoidable and our bound is asymptotically optimal (for any fixed $n$) for safe action model learning for domains with conditional effects:


\begin{theorem}\label{thm:lower-bound}
Any learning algorithm that is guaranteed to return a safe action model must be given at least $m\geq\Omega(\frac{1}{\epsilon}(|F||A||(|F|/3n)^{n}|+\log\frac{1}{\delta}))$ samples to be able to guarantee that with probability at least $1-\delta$ the learned model permits a plan solving $\Pi$ drawn from $\mathcal{D}$ with probability at least $1-\epsilon$ for $0<\epsilon,\delta < 1/4$.
\end{theorem}

At a high level, the hard distribution involves initial states that have all $|A|$ ``goal'' fluents set to false, all but one (uniformly random) of the $(p-|A|)/2$ ``forbidden'' fluents true, and exactly $n$ out of the  $(p-|A|)/2$ ``flag'' fluents (uniformly at random) true. With probability $4\epsilon$, the goal includes a single goal fluent, chosen uniformly at random, that should be set to true. All other goal fluents, as well as the one forbidden fluent, must be set to false. The corresponding $\Pi$ consists of a plan with a single action, where the agent takes an action corresponding to the fluent to be set true in the goal. Otherwise, there is an empty goal, where the agent takes a no-op action. For any problem with a non-empty goal that we did not observe in the training set, no safe action model can permit taking the action needed to achieve the goal. We need to observe at least a $3/4$ fraction of the possible goals for a safe action model to attain probability $1-\epsilon.$

\begin{proof}
For any $p\geq 3|A|$, consider a domain in which there is a no-op action with no effects, and for each other action $a_i\in A$ 
there is a ``goal'' fluent $f_{i}$ that is the effect of exactly one action, and this is the only effect. 
The domain includes an additional set of $(p-|A|)/2$ ``flag'' fluents, and $(p-|A|)/2$ ``forbidden'' fluents (so there are $|A|+2(p-|A|)/2=p$ fluents in total).

Now consider the following distribution $\mathcal{D}$ on problems and plans. 
The initial states have all goal fluents set to false, all but one (uniformly random) forbidden fluent true, and exactly $n$ of the flag fluents (uniformly at random) true. 
With probability $1-4\epsilon$, the goal is empty.
Otherwise, the goal includes a single goal fluent, chosen uniformly at random, that should be set to true. All other goal fluents, as well as the one forbidden fluent, must be set to false. 
The corresponding $\Pi$ always consists of a plan with a single action; for the empty goal, the agent takes the no-op action, and otherwise the agent takes the action corresponding to the $f_{i}$ goal fluent to be set true in the goal. 

For any problem with a non-empty goal that we did not observe in the training set, the action model that is obtained from the real action model by adding the forbidden fluent as a conditional effect of the corresponding goal action with the flag fluents as the condition, is consistent with the training set. 
Indeed, either the action appears with a different set of flags so that one of the flag fluents in this condition is falsified (and the corresponding effect does not occur), a different forbidden fluent is false (so the relevant forbidden fluent is already true and the effect is not observed), or else the action differs from the one we need to achieve this goal, and then the effect is identical to the real action model. Therefore, no safe action model can permit taking the action needed to achieve the goal, and all other actions would reach a state in which some incorrect goal fluent is set to true and cannot be subsequently set to false. 

Since the no-op goal only comprises $1-4\epsilon$ probability in the goal distribution, we need to observe at least a $3/4$ fraction of the possible goals for a safe action model to attain probability $1-\epsilon.$ But, there are $|A|$ goals, $(p-|A|)/3\geq p/3$ forbidden fluents, and ${(p-|A|)/3\choose n}\geq\left(\frac{p}{3n}\right)^n$ sets of flags, and in expectation, a sample of size $m$ only contains $4\epsilon m$ examples of these pairs of goals and flag settings. We, therefore, need $\Omega\left(\frac{1}{\epsilon}(|F|/3)^n|F||A|\right)$ examples; likewise, to even observe any of the nonempty goals with probability $1-\delta$, we need $\Omega\left(\frac{1}{\epsilon}\log\frac{1}{\delta}\right)$ examples, giving the claimed bound.
\end{proof}


\section{Learning Lifted Action Models}
Defining PDDL domains and problems in a \emph{lifted} manner is common.  
A lifted domain defines fluents and actions in a parameterized manner,  where every parameter has a \emph{type}. 
For example, the action \emph{(stop ?f - floor)} and the fluent \emph{(destin ?person - passenger ?floor - floor)} from the IPC Miconic domain are parameterized by objects of type \emph{floor} and \emph{person}. 
A state is a conjunction of \emph{grounded fluents}, which are pairs of the form $\tuple{l,b_l}$ where $l$ is a fluent, and $b_l$ is a function that maps parameters of $l$ to concrete objects. 
A plan is a sequence of \emph{grounded actions}, which are pairs in the form $\tuple{a,b_a}$ where $a$ is an action and $b_a$ maps action parameters to objects. 
A trajectory is an alternating sequence of states and grounded actions. 

Generally, the parameters in an action's preconditions and effects are bound to the action's parameters. 
Thus, preconditions and effects of an action in a lifted domain are \emph{parameter-bound literals}. 
A parameter-bound literal for an action $a$ is a pair $(l,b_{la})$ where $l$ is a literal and $b_{la}$ is a function that maps every parameter of $l$ to a parameter in $a$. 
Let $\bindings(a)$ be the function that returns all parameter-bound literals that can be bound to $a$. For a grounded action $a_G=\tuple{a, b_a}$ and parameter-bound literal $l\in \bindings(a)$, we define
$\grounding(a_G, l)$ to be the grounded literal resulting from assigning the objects in the parameters of $a_G$ to the parameters of $l$. 
Given a conjunction of parameter-bound literals $c$, $\grounding(a_G, c)$ returns the corresponding conjunction of grounded literals $c_G$ such that $\forall l\in c: \grounding(a_G, l)\in c_G$.
Similarly, for a pair of conjunctions of parameter-bound literals $(c,e)$ we define $\grounding(a_G,c,e)$ to be the pair $(c_G,e_G)$ that
are the corresponding conjunctions of grounded literals. 
SAM learning has already been extended to learn lifted classical planning domains~\citep{juba2021safe} without conditional effects. 
We extend \consam to support lifted domains in a similar manner, based on the following extension to the \consam inductive rules (Def.~\ref{def:basic-rules}).
\begin{definition}[Lifted \consam Inductive Rules]\label{def:lifted-rules}
For every action triplet $\tuple{s,a_G=\tuple{a,b_a}, s'}\in \mathcal{T}$,
\begin{compactenum}
    \item \,[Not a precondition] For every $l\in \bindings(a)$ s.t.\ $\grounding(a_G,l)\notin s$, $l\notin \pre(a)$ 
    \item \,[Not a result] For every $l'\in \bindings(a)$ s.t.\  
    $\grounding(a_G,l')\notin s'$, 
    $\nexists (c,e)\in\eff(a)$ where 
    $(\grounding(a_G,c)\wedge s\nvdash\bot)\wedge(l'\in e)$
    \item \,[Must be an effect] For every $l'\in \bindings(a)$, if $\grounding(a_G,l') \in s'\setminus s$ 
    then $\exists (c,e)\in \eff(a)$, 
    where $(\grounding(a_G,c) \wedge s\nvdash\bot) 
                  \wedge  (l'\in e)$        

      \item \,[Not an antecedent] For every $l'\in \bindings(a)$  
      and set of literals $c\subseteq \bindings(a)$
      if $\grounding(a_G,l') \in s'\setminus s$
      and $\grounding(a_G,c)\wedge s\vdash\bot$       
      then $\nexists(c',e)\in \eff(a)$ such that $c \subseteq c'$ and $l'\in e$
\end{compactenum}
\end{definition}
The rest of the \consam algorithm remains essentially the same, where \mustber\ and \possiblya\ may now contain parameter-bound literals.
\section{Learning Effects with Universal Quantifiers}
Some PDDL domains and planners support \emph{universal quantifiers}, which allow actions' preconditions and effects to include additional parameters that are not bound to the actions' parameters. 
More formally, \emph{Universally quantified} preconditions and effects define one or more \textit{\uqv} that may be bound to any parameter of a literal used in them. 
The result of universally quantified conditional effects \emph{must} include at least one \uqv. 
Otherwise, if only the antecedents include \uqvs, then we can interpret such antecedents as disjunctive universal preconditions.
For example, suppose we need to represent an elevator with a stopping functionality that ensures all waiting passengers get in or out once the elevator stops. 
Figure~\ref{list:universal-effects-action} presents the \textit{stop} action schema to implement this functionality. The \uqv in this example is $?p$.
\begin{figure}[ht]
\begin{center}
\begingroup
    \fontsize{8pt}{8pt}\selectfont
\begin{Verbatim}[commandchars=\\\{\}]
(\textcolor{blue}{:action} stop
\textcolor{blue}{:parameters} (?f - floor)
\textcolor{blue}{:precondition} (and (lift-at ?f))
\textcolor{blue}{:effect} 
(and (\textcolor{blue}{forall} (?p - passenger)
    (when
      (and (boarded ?p) (destin ?p ?f))
      (and (not (boarded ?p)) (served ?p))))))
\end{Verbatim}
\endgroup
\caption{Parts of the action \textit{stop} from Miconic domain that contains universally conditional effects.}
\label{list:universal-effects-action}
\end{center}
\end{figure}

We focus below on learning universal effects since they are more common, but our approach also supports learning universal preconditions.
Note that universal effects may be unconditional and occur every time the action is executed, in which case the antecedent is the trivial antecedent $\true$. 
\consam can learn these universal effects as well.
We briefly describe how \consam can be extended to support universal effects. 

In general, the number of \uqvs a universal effect can define is exponential in the arity of the domain fluents. Still, universal effects with more than two \uqv are rare. 
Thus, we will assume the number of \uqvs in a universal effect is a known fixed constant $k$.  
To support universal effects, the $\bindings(a)$ function is modified to also return parameter-bound literals that bind one or more literal parameters to \uqvs that may be used in action $a$'s effects. 
Similarly, the $\grounding(a_G,l)$ function is modified such that if $l$ is a parameter-bound literal that includes \uqvs
then $\grounding(a_G,l)$ returns a \textit{set} of grounded literals matching the grounded action's parameters combined with the \uqvs.
In addition, $\grounding(a_G,c,e)$ returns a set of matching pairs $(c_G,e_G)$ if either $c$ or $e$ include one or more \uqvs. 
We now present the changes in the inductive rules to support universally quantified variables.
\begin{definition}[\consam Inductive Rules with \uqvs]
\label{def:lifted-rules-univ}
For every action triplet $\tuple{s,a_G=\tuple{a,b_a}, s'}\in \mathcal{T}$:
\begin{compactenum}
    \item For every $l\in \bindings(a)$ such that $\exists l_G\in\grounding(a_G,l)$ where $l_G\notin s$, then $l\notin \pre(a)$\footnote{The first inductive rule enables learning universal preconditions} 

    \item For every $l'\in \bindings(a)$ such that $\exists l'_G \in \grounding(a_G,l')$ and $l'_G\notin s'$ then
    $\nexists (c,e)\in\eff(a)$ such that $\exists (c_G,e_G)\in \grounding(a_G,c,e)$ where $(c_G\wedge s\nvdash\bot) \wedge (l'_G\in e_G)$

    \item For every $l'\in \bindings(a)$ if $\exists l'_G\in \grounding(a_G,l')$ such that $l'_G \in s'\setminus s$ 
    then $\exists (c,e)\in \eff(a)$, 
    where $\exists(c_G,e_G)\in \grounding(a_G,c,e)$ such that $ c_G \wedge s\nvdash\bot \wedge l'_G\in e_G$        

    \item  For every $l'\in \bindings(a)$ and $c\subseteq \bindings(a)$
      if $\exists(l'_g,c_G)\in \grounding(a_G,c,l')$ such that $l'_G \in s'\setminus s$
      and $c_G\wedge s\vdash\bot$       
      then $\nexists(c',e)\in \eff(a)$ such that $c \subseteq c'$ and $l'\in e$
\end{compactenum}
\end{definition}

Next, we presented the \consam algorithm that supports lifted actions with \uqvs. One key assumption we make is the \emph{inductive binding assumption}, which means that 
for every transition $\tuple{s,a_G=\tuple{a,b_a}, s'}$ 
and grounded literal $l'_G\in s'$ 
there exists a single parameter-bound literal $l'$ 
that satisfies $l'_G\in g(a_G,l')$. How to apply SAM in domains without this assumption is significantly more complex, and is explained by Juba et al.~\citet{juba2021safe}.

The algorithm starts similarly to its grounded version, except that instead of initializing the preconditions with the grounded literals, it initializes the data structures with their parameter-bound counterparts that might contain \uqvs. 
The algorithm applies lifted inductive rules and then compiles the final preconditions and effects using \LearnConditions{} (Algorithm~\ref{alg:lifted-learn-conditionals}).
\LearnConditions{}'s main difference compared to its grounded version is that the function iterates over every $u\in\uqvs$, adding a universal precondition to the action if necessary. 
The pseudo-code for the algorithm is presented in Algorithm~\ref{alg:lifted-consam-algorithm}.

\begin{algorithm}[ht]
\small
\caption{\consam Algorithm -- Lifted Version}\label{alg:lifted-consam-algorithm}
\begin{algorithmic}[1]
\State \textbf{Input}: $\mathcal{T}, n$ 
\State \textbf{Output}: A safe action model.
\For{$a \in A_\lifted(\mathcal{T})$}
    \State $\pre(a) \gets L_\pb(\mathcal{T})$
    \State $\mustber(a) \gets \emptyset$ 
    \State $\possiblya(l,a) \gets \bigcup_{i=1}^n \{l_1\wedge ...\wedge l_i | \forall 1\leq j\leq i: l_j\in L_\pb(\mathcal{T}) \} \cup \{\true\}$
\EndFor
\For{$\tuple{s,a_G=\tuple{a,b_a}, s'}\in \mathcal{T}$} \label{alg:applying-rules} 
    \For{$l$ such that $\exists l_G\in\grounding(a_G,l)$ where $l_G\notin s$} 
        \State $pre(a) \gets pre(a) \setminus \{l\}$\label{alg:rule1}  \Comment{Rule \ref{rule:not-a-precondition}}
    \EndFor
    \For{$l'_G\in s'\setminus s$} \Comment{Rule \ref{rule:must-be-effect}}
        \State $l'\gets$ select from $\{l'| l'\in\bindings(a), l'_G\in g(a_G,l')\}$ \label{alg:injective-binding}
        \State $\mustber(a)$ $\gets$ $\mustber(a) \cup \{l'\}$
    \EndFor   
    \For{$l'_G \notin s'$} 
        \State $l'\gets$ select from $\{l'| l'\in\bindings(a), l'_G\in g(a_G,l')\}$
        \For{$c$ s.t $\exists (c_G,l'_G)\in \grounding(a_G,c,l')$ where $(c_G\wedge s\nvdash\bot)$}
            \State $\possiblya(l',a) \gets \possiblya(l',a)\setminus c$
        \EndFor 
    \EndFor \Comment{Rule \ref{rule:not-an-effect}}
    \For{$l'_G \in s'\setminus s$} \Comment{Rule \ref{rule:no-disjunctive}}
        \State $l'\gets$ select from $\{l'| l'\in\bindings(a), l'_G\in g(a_G,l')\}$
        \For{$c$ s.t $\exists (c_G,l'_G)\in \grounding(a_G,c,l')$ where $(c_G\wedge s\vdash\bot)$}
            \State $\possiblya(l',a) \gets \possiblya(l',a)\setminus c$
        \EndFor 
    \EndFor        
\EndFor
\State \Return  \textbf{\LearnConditions}$(\pre,\mustber,\possiblya)$
\end{algorithmic}
\end{algorithm}

\begin{algorithm}[ht]
\small
\caption{\LearnConditions{} -- Lifted Version}\label{alg:lifted-learn-conditionals}
\begin{algorithmic}[1]
\State \textbf{Input}: $\pre,\mustber,\possiblya$ 
\State \textbf{Output}: $\pre^*$ and $\eff^*$ for all actions.
\For{$a \in A_\lifted(\mathcal{T})$}
    \State $\pre^*(a) \gets \emptyset$
    \For{$u\in \uqvs(a)$}
        \State $\eff(a) \gets \emptyset$; ~~ $\pre_u^*(a) \gets \bigwedge_{l\in \pre(a)}l$
        \For{$l\in L_\pb(\mathcal{T})\setminus\pre(a)$ where $\possiblya(l,a)\neq\emptyset$} 
            \State $PA\gets \{c\in \possiblya(l,a)|(\pre(a)\cap c) = \emptyset\}$
            \State $\mustnotbea \gets \bigwedge_{c\in PA}\neg c$
            \State $\mustbea \gets \bigwedge_{c\in PA} c$
            \State Minimize $\mustbea$ and $\mustnotbea$ using unit propagation.
            \If{$l \in \mustber(a)$}
                \State Add to $\eff(a)$: $(\mustbea, l)$ 
                \If{$PA$ is not a single clause}
                    \State $\pre_u^*(a) \gets \pre_u^*(a) \wedge (l\vee \mustnotbea \vee \mustbea)$ 
                \EndIf          
            \Else
                \State $\pre_u^*(a) \gets \pre_u^*(a) \wedge (l\vee \mustnotbea)$ 
            \EndIf
        \EndFor
        \State $\pre^*(a) \gets \pre^*(a) \wedge \pre_u^*(a)$
    \EndFor
\EndFor
\State \Return  $\tuple{\pre^*, \eff}$
\end{algorithmic}
\end{algorithm}

\section{Experimental Results}
We implemented \consam{} and conducted experiments on six planning domains, including conditional effects. 
Specifically, we used the CityCar, Nurikabe, and Maintenance domains from the International Planning Competition (IPC) 2014~\citep{vallati20152014}; the Briefcase and Miconic are from AIPS-2000~\citep{bacchus2001aips}, and Satellite, which is an ADL version of the classical IPC~\citep{long20033rd} domain~\footnote{All the domains are available in https://github.com/AI-Planning/classical-domains}. There are more domains containing conditional effects available. Table~\ref{tab:my-table} presents the available domains and their properties. The experimented domains are those in which all the assumptions hold. 
We summarize the properties we used to classify the domains: 
(1) Disjunctive antecedents - i.e., whether a conditional effect can appear in more than one "when'' clause. 
(2) Existential antecedents / preconditions - whether the domain's actions contain existential quantifiers in the preconditions or the effects. 
(3) Object equality in antecedents / preconditions - whether the antecedents contained statements in which two parameters must be equal.
(4) Type hierarchy - whether the domain has nested types. 
(5) Implications - whether the actions contain implications in their preconditions. 
(6) Planners can solve with the real model - whether the planners used in our experiments were able to solve the datasets' problems. 

\consam does not support object equality checks in the antecedents thus, domains containing these conditions were not used in the experiments. 
Finally, type hierarchy is also not supported by \consam since it can lead to ambiguities in predicate matching.


\begin{table*}[tb]
\centering
\resizebox{\textwidth}{!}{%
\begin{tabular}{|c|c|c|c|c|c|c|c|}
\hline
\textbf{Domain} & \textbf{\begin{tabular}[c]{@{}c@{}}Used \\ in experiments\end{tabular}} & \textbf{\begin{tabular}[c]{@{}c@{}}Planners can solve \\ with the real model\end{tabular}} & \textbf{\begin{tabular}[c]{@{}c@{}}Disjunctive \\ antecedents \\ (Assumption 4)\end{tabular}} & \textbf{\begin{tabular}[c]{@{}c@{}}Existential \\ preconditions / effects\end{tabular}} & \textbf{\begin{tabular}[c]{@{}c@{}}Object equality \\ in antecedents / \\ precondition\end{tabular}} & \textbf{\begin{tabular}[c]{@{}c@{}}Type \\ hierarchy\end{tabular}} & \textbf{Implications} \\ \hline
Airport-adl & X & - & X & V & X & X & V \\ \hline
Briefcase & V & V & X & X & X & X & X \\ \hline
Caldera & X & - & X & X & V & V & X \\ \hline
Cavediving & X & X & X & X & X & X & X \\ \hline
Citycar & V & V & X & X & X & X & X \\ \hline
Elevators & V & V & X & X & X & X & X \\ \hline
maintenance & V & V & X & X & X & X & X \\ \hline
Nurikabe & V& V & X & X & X & X & X \\ \hline
Schedule & X & - & X & X & V & X & X \\ \hline
Satellite & V & V & X & X & X & X & X \\ \hline
Settlers & X & - & X & X & V & V & X \\ \hline
Spider & X & -& X & X & X & V & X \\ \hline
\end{tabular}%
}
\caption{Detailed information about the domains containing conditional effects available in the domains repository - https://github.com/AI-Planning/classical-domains}.
\label{tab:my-table}
\end{table*}

Table~\ref{tab:domain-stats} presents additional information about the domains that were used 
 in out experiments.
The column `Domain` represents the domain's name that was experimented on, and the columns `$|A|$` and `$|F|$` present the number of lifted actions and fluents in the domain respectively. The column `\# U.E.` presents the number of universally conditional effects present in the domains (The satellite domain has conditional effects that do not contain \uqvs) and the column `$n$` is the maximal number of antecedents for the conditional effects in the domains. The column $|\mathcal{T}|$ represents the size of the trajectories dataset of the domain.
The column $|t|$ is the average number of action triplets in a trajectory (the standard deviation is displayed in brackets).

We generated our problems dataset for each domain using a PDDL problem generator~\citep{seipp-et-al-zenodo2022}.\footnote{Action costs were ignored in all domains we experimented on.} 
Using the problem generator, we created a dataset of 100 problems that were used to create the trajectories. 
To solve the generated problems and create the input trajectories, we used two well-known classical planners that support ADL, Fast-Downward (FD)~\citep{helmert2006fast} using FF heuristic and context-enhanced additive heuristic, and Fast-Forward (FF)~\cite{hoffmann2001ff} with a Greedy BFS configuration.
We restricted the solvers to solve the problems in up to 60 seconds.
For the Nurikabe domain, only 52 problems were solved with our planners, resulting in a smaller dataset.

We split our dataset into train and test sets, trained \consam on the trajectories in the train set, and used the generated action models to solve the test set problems. 
We used VAL~\citep{howey2004val} to validate the generated plans' correctness.
We followed a 5-fold cross-validation methodology by repeating each experiment 5 times, sampling different trajectories for learning and testing. All the presented results are averaged over the five folds. 
The experiments were run on a Linux machine with 8 cores and 16 GB of RAM.

\begin{table}[tb]
\centering
\resizebox{\columnwidth}{!}{%
\begin{tabular}{@{}lcccccl}
\toprule
\textbf{Domain} & \multicolumn{1}{l}{\textbf{$|A|$}} & \multicolumn{1}{l}{\textbf{$|F|$}}  & \multicolumn{1}{l}{\textbf{\# U.E.}} & \multicolumn{1}{l}{\textbf{$n$}} & $|\mathcal{T}|$ & $|t|$\\ \midrule
Satellite & 5 & 8 & 0 & 1 & 100 & 36.2 (6.1)\\
Maintenance & 1 & 3 & 1 & 1 & 100 & 5.6 (1.3)\\
Miconic & 3 & 6 & 2 & 2 & 100 & 46.7 (4.9)\\  
Citycar & 7 & 10 & 1 & 1 & 100 & 19.7 (4.0)\\
Briefcase & 3 & 3 & 1 & 1 & 100 & 76.6 (15.6) \\
Nurikabe& 4& 12& 2& 3& 52& 78.2 (5.17)\\\bottomrule
\end{tabular}%
}
\caption{Statistics regarding the experimented domains.}
\label{tab:domain-stats}
\vspace{-0.6cm}
\end{table}

\subsection{Evaluation Metrics}
We evaluated our algorithm using two metrics: the percentage of the test set problems solved using \consam's learned model, and the correctness of the learned model using precision and recall measures. 
A test set problem is regarded as solved if one of the planners we used (FD and FF) was able to solve it with the learned action model. 
Measuring the \textit{syntactic} precision and recall of the learned model, i.e., measuring the textual difference between the real and learned domains, may not represent the usefulness of the learned domain in solving problems. 
Instead, we measure the \textit{semantic} precision and recall of the learned model's preconditions and effects, as follows. 
For each state in our trajectories, we try to apply the actions of the learned and real action models on the state. 
We measure the precision and recall based on which action is applicable in the tested states. Formally, the semantic precision and recall of the preconditions are:
\begin{align*}\small
    P^{sem}_{\pre}(a)= \frac{|app_{\realm}(a)\cap app_M(a)|}{|app_M(a)|}  \\
    R^{sem}_{\pre}(a)= \frac{|app_{\realm}(a)\cap app_M(a)|}{|app_{\realm}(a)|}
\end{align*}
Where $app_M(a)$ denotes the states in a set of trajectories where $a$ is applicable according to the action model $M$. 
Finally, average the results over all actions of the domain.

Since \consam learns a safe action model, the semantic precision of the preconditions is always one. 
Furthermore, the \consam's safety property indicates that whenever an action is applicable according to \consam its effects are identical to the real domain's effects. 
Thus the precision and recall of the effects is always one as well. 
Thus, our evaluation below only presents semantic recall of the actions' preconditions.

\begin{figure*}[ht]
    \centering
    \includegraphics[width=\textwidth]{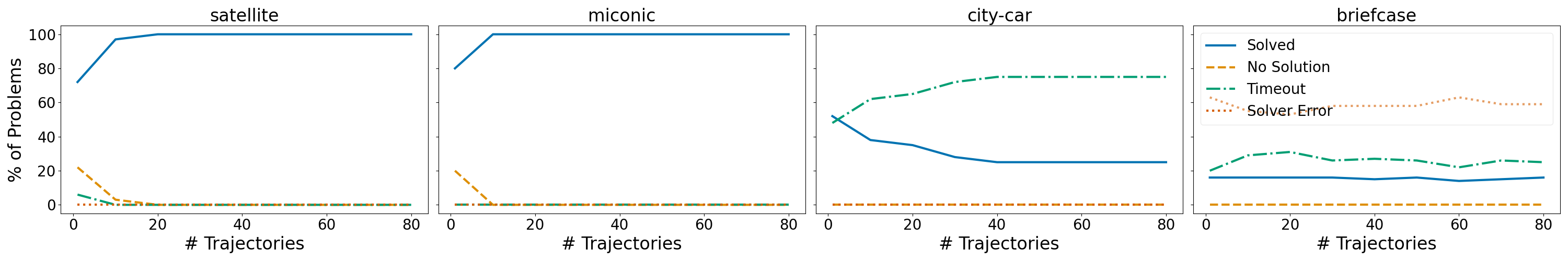}
    \caption{Solving statistics of Satellite, CityCar, Briefcase and Miconic domains.}
    \label{fig:combined-plots}
\end{figure*}

\subsection{Results}
Table~\ref{tab:results} displays the experimental results with the maximal number of trajectories given as input. 
The column $\%S$ represents the percent of the problems that were solved by the planners, 
$\%TO$ represents the percentage of problems in which the planner had timed out (i.e., reached our 60-second time limit), 
and $\%NS$ represents the percent of problems that were declared unsolvable with the learned domain. 
This is caused when the learned domain is too restrictive and thus some problems cannot be solved with it.
The column $\%ERROR$ represents the percentage of problems in which the solver encountered an error while solving the test set problems. 
Such errors were caused when the planner was killed due to extensive resource consumption.
The column \textit{Planner} represents the planner that had the best performance and their results are being presented.
Finally, the column $R_{pre}$ denotes the preconditions' semantic recall.

\begin{table}[tb]
\centering
\resizebox{\columnwidth}{!}{%
\begin{tabular}{@{}lccccll}
\toprule
\textbf{Domain} & \multicolumn{1}{l}{\textbf{$\%S$}} & \multicolumn{1}{l}{\textbf{$\%TO$}}  & \multicolumn{1}{l}{\textbf{$\%NS$}} & \multicolumn{1}{l}{\textbf{$\%ERROR$}} &  Planner&$R_{sem}$\\ \midrule
Satellite & 100& 0& 0& 0 & FD / FF&0.99\\
Maintenance & 100& 0& 0& 0 & FF&1.00\\
Miconic & 100& 0& 0& 0& FF / FD& 1.00\\  
Citycar & 25& 75& 0& 0 & FD& 0.99\\
Briefcase & 16 & 25& 0& 59 & FF & 1.00\\
Nurikabe& -& -& -& -& -&-\\\bottomrule
\end{tabular}%
}
\caption{Experimental results for \consam.}
\label{tab:results}
\vspace{-0.6cm}
\end{table}

For the Satellite, Maintenance, and Miconic domains, all the test set problems were solved perfectly using the domain learned with \consam. 
The results for CityCar and Briefcase are significantly worse, where the percent of problems that were solved are 25\% and 19\%, respectively. 
A possible explanation for these results is that the learned domain returned by \consam{} may contain complex universal preconditions. These preconditions affect the solvers in their ability to solve the test set problems. 
Indeed, every problem that was not solved in CityCar domain was not solved due to our timeout restrictions. 
Similarly, in the Briefcase domain, the planning process terminated on many occasions because it consumed too many resources (expressed in the $\%ERROR$ column). 
However, the calculated preconditions' semantic recall for both domains is 0.99. 
That suggests that while the learned domains may appear to be more complex than their original counterpart, they are nearly semantically identical. 
For the Nurikabe domain, \consam{} could not solve any test problem with the learned domain. This may be because this domain has the largest antecedents ($n=3$). 

Figure~\ref{fig:combined-plots} presents the solving statistics as a function of the number of trajectories used to train \consam for the domains Satellite, Miconic, CityCar, and Briefcase. 
We do not present the Maintenance domain results graphically since all the test set problems were solved perfectly after a single trajectory. 
For the Satellite and Miconic domains, the percentage of the solved problems increases monotonically with the number of examples. On the other hand, in the CityCar domain, we observe a decrease in the number of solved problems as the number of trajectories increases. 
This is because with single trajectory, the domain \consam learns does not contain the action \textit{destroy\_road}, which has universal effects. 
Without this action the domain is less complex and the planners could solve more test set problems. 
Once the action is learned, the domain becomes much more complex and thus the planners timeout more often. 
Finally, the poor results observed in the Briefcase domain were due to insufficient resources with 59\% of the test set problems not being solved since the planning process was killed due to high resource consumption. 

\section{Supporting Disjunctive Antecedents}
We focused on conditional effects where the result can appear only once in each action (Assumption~\ref{ref:assumption2}). 
There are cases where such an assumption does not hold. For example, in our flu treatment action imagine that now the allergic reaction can appear if the patient has a rare blood type or if they are sleep deprived. In this case, our action would look as follows:
\begin{figure}[ht]
\begin{center}
\begingroup
    \fontsize{8pt}{8pt}\selectfont
\begin{Verbatim}[commandchars=\\\{\}]
(\textcolor{blue}{:action} treat-flu-symptoms-X
\textcolor{blue}{:parameters} (?p - patient ?b_type - bloodType)
\textcolor{blue}{:precondition} (and (has-flu ?p))
\textcolor{blue}{:effect} 
(and (when (is-rare-blood-type ?b_type) 
           (allergic-reaction ?p))
     (when (sleep-deprived ?p) 
           (allergic-reaction ?p))  
     (and (not (has-flu ?p)))))
\end{Verbatim}
\endgroup
\vspace{-0.3cm}
\caption{An action representing a flu medicine with a conditional effect.}
\vspace{-0.5cm}
\label{list:mission-critical-domain-with-disjunction}
\end{center}
\end{figure}

In Figure~\ref{list:mission-critical-domain-with-disjunction} we present the action \textit{treat-flu-symptoms-X} that contains disjunctive antecedents. Supporting disjunctive antecedents requires altering the first assumption to address only the actions' preconditions and completely removing the second assumption.
Removing the second assumption affects \consam's fourth inductive rule since it no longer holds that if a conjunction of literals does not hold in a state, it cannot be an antecedent of the observed result. That is since now conditional effects might be disjunctive. 

Supporting the new capability requires a minor change to the \consam algorithm. We change the initialization process of \possiblya{} to have every possible CNF clause with up to $n$ antecedents, i.e., now \possiblya{} includes CNFs and not just conjunctions of literals. The rest of the algorithm is not affected. We note that this change highly increases the algorithm's complexity since now it has to eliminate every possible CNF expression before it can determine the correct set of antecedents.
Note that the available benchmark domains do not contain disjunctive antecedents for conditional effects. 
Furthermore, due to its prohibitive complexity, we decided not to support disjunctive antecedents in \consam and leave this functionality for future work.


\section{Related Work}
Several prior works learn action models from trajectories.
The Action-Relation Modelling System (ARMS)~\citep{yang2007learning} algorithm learns a PDDL description of action models by extracting a set of weighted constraints from the input plan examples. 
The Simultaneous Learning and Filtering (SLAF)~\citep{amir2008learning} algorithm is a different algorithm for learning action models designed for partially observable deterministic domains. 
The Learning Object-Centred Models (LOCM, LOCM2)~\citep{cresswell2013acquiring,cresswell2011generalised} is another action model learning algorithm that analyzes plan sequences, where each action appears as an action name and arguments in the form of a vector of object names.
FAMA~\cite{aineto2019learning} is a state-of-the-art algorithm that learns action models with minimal state and action observability.
FAMA can learn from gapped action sequences of actions, and in the extreme, FAMA can even learn when only given the initial and the final states as input.

The algorithms presented above learn action models that do not guarantee that the actions learned are applicable according to the agent's actual action model definition. 
Contrary to these algorithms, the \sam family of algorithms is designed to learn action models in a setting where execution failures must be avoided~\citep{stern2017efficientAndSafe,juba2021safe,juba2022learning}. To this end, \sam generates a conservative action model. Planning with such an action model produces \textit{sound} plans but may fail to find a plan even if such exists (i.e., it is \textit{incomplete}). 

To the best of our knowledge, there is no work focusing on learning safe action models with conditional effects. 
~\citet{oates1996learning} created an algorithm that can learn planning operators for STRIPS~\citep{fikes1971strips} by interacting with the environments and performing random actions, and using search techniques to learn the context-dependent operators. This approach uses random walks which are costly in case the agent cannot recover from failures. Furthermore, the resulting action model generated is grounded while our approach learns a lifted PDDL domain.
~\citet{zhuo2010learning} focused on learning action models with quantifiers and implications, and proved that their algorithm could learn simple conditional effects with only one antecedent. They aimed to \textit{reduce} the domain compilation time for domain experts. Thus, the domains their algorithm outputs may be incomplete or even wrong. This means that their algorithm does not work in mission-critical settings.

\section{Conclusions and Future Work}
In this work, we presented \consam, an algorithm that can learn action models for domains that include conditional and universal effects. 
We showed that \consam learns a safe action model w.r.t. the real unknown action model and runs in reasonable time. Moreover, we presented tight sample complexity results, showing that \consam{} is, in a sense, asymptotically optimal. 
Our experimental results show that using a small number of trajectories, \consam learns an action model that solves the test set problems. 
In future works, we aim to explore methods to improve the algorithm's scalability and support domains with more expressive conditional effects that might contain unbounded disjunctive antecedents or even include numeric conditions and effects.


\bibliography{camera_ready}

\begin{thebibliography}{25}
\providecommand{\natexlab}[1]{#1}

\bibitem[{Aineto, Celorrio, and Onaindia(2019)}]{aineto2019learning}
Aineto, D.; Celorrio, S.~J.; and Onaindia, E. 2019.
\newblock Learning action models with minimal observability.
\newblock \emph{Artificial Intelligence}, 275: 104--137.

\bibitem[{Aineto, Jim{\'e}nez, and Onaindia(2018)}]{aineto2018learning}
Aineto, D.; Jim{\'e}nez, S.; and Onaindia, E. 2018.
\newblock Learning STRIPS action models with classical planning.
\newblock In \emph{International Conference on Automated Planning and Scheduling (ICAPS)}.

\bibitem[{Amir and Chang(2008)}]{amir2008learning}
Amir, E.; and Chang, A. 2008.
\newblock Learning partially observable deterministic action models.
\newblock \emph{Journal of Artificial Intelligence Research}, 33: 349--402.

\bibitem[{Bacchus(2001)}]{bacchus2001aips}
Bacchus, F. 2001.
\newblock AIPS 2000 planning competition: The fifth international conference on artificial intelligence planning and scheduling systems.
\newblock \emph{Ai magazine}, 22(3): 47--47.

\bibitem[{Cresswell and Gregory(2011)}]{cresswell2011generalised}
Cresswell, S.; and Gregory, P. 2011.
\newblock Generalised domain model acquisition from action traces.
\newblock In \emph{International Conference on Automated Planning and Scheduling (ICAPS)}, 42--49.

\bibitem[{Cresswell, McCluskey, and West(2013)}]{cresswell2013acquiring}
Cresswell, S.; McCluskey, T.; and West, M. 2013.
\newblock Acquiring planning domain models using LOCM.
\newblock \emph{The Knowledge Engineering Review}, 28(2): 195--213.

\bibitem[{Fikes and Nilsson(1971)}]{fikes1971strips}
Fikes, R.~E.; and Nilsson, N.~J. 1971.
\newblock STRIPS: A new approach to the application of theorem proving to problem solving.
\newblock \emph{Artificial Intelligence}, 2(3-4): 189--208.

\bibitem[{Ghallab et~al.(1998)Ghallab, Howe, Knoblock, McDermott, Ram, Veloso, Weld, and Wilkins}]{aeronautiques1998pddl}
Ghallab, M.; Howe, A.; Knoblock, C.; McDermott, D.; Ram, A.; Veloso, M.; Weld, D.; and Wilkins, D. 1998.
\newblock PDDL -- The Planning Domain Definition Language.
\newblock \emph{Technical Report, Tech. Rep.}

\bibitem[{Ghallab, Nau, and Traverso(2016)}]{ghallab2016automated}
Ghallab, M.; Nau, D.; and Traverso, P. 2016.
\newblock \emph{Automated planning and acting}.
\newblock Cambridge University Press.

\bibitem[{Helmert(2006)}]{helmert2006fast}
Helmert, M. 2006.
\newblock The Fast Downward planning system.
\newblock \emph{Journal of Artificial Intelligence Research}, 26: 191--246.

\bibitem[{Hoffmann(2001)}]{hoffmann2001ff}
Hoffmann, J. 2001.
\newblock FF: The fast-forward planning system.
\newblock \emph{AI magazine}, 22(3): 57--57.

\bibitem[{Howey, Long, and Fox(2004)}]{howey2004val}
Howey, R.; Long, D.; and Fox, M. 2004.
\newblock VAL: Automatic plan validation, continuous effects and mixed initiative planning using PDDL.
\newblock In \emph{16th IEEE International Conference on Tools with Artificial Intelligence}, 294--301. IEEE.

\bibitem[{Juba, Le, and Stern(2021)}]{juba2021safe}
Juba, B.; Le, H.~S.; and Stern, R. 2021.
\newblock Safe Learning of Lifted Action Models.
\newblock In \emph{International Conference on Principles of Knowledge Representation and Reasoning ({KR})}, 379--389.

\bibitem[{Juba and Stern(2022)}]{juba2022learning}
Juba, B.; and Stern, R. 2022.
\newblock Learning Probably Approximately Complete and Safe Action Models for Stochastic Worlds.
\newblock In \emph{AAAI Conference on Artificial Intelligence}.

\bibitem[{Long and Fox(2003)}]{long20033rd}
Long, D.; and Fox, M. 2003.
\newblock The 3rd international planning competition: Results and analysis.
\newblock \emph{Journal of Artificial Intelligence Research}, 20: 1--59.

\bibitem[{Mordoch, Stern, and Juba(2023)}]{mordoch2023learning}
Mordoch, A.; Stern, R.; and Juba, B. 2023.
\newblock Learning Safe Numeric Action Models.
\newblock In \emph{{AAAI} Conference on Artificial Intelligence (AAAI)}.

\bibitem[{Oates and Cohen(1996)}]{oates1996learning}
Oates, T.; and Cohen, P.~R. 1996.
\newblock Learning planning operators with conditional and probabilistic effects.
\newblock In \emph{Proceedings of the AAAI Spring Symposium on Planning with Incomplete Information for Robot Problems}, 86--94.

\bibitem[{Pednault(1989)}]{pednault1989adl}
Pednault, E.~P. 1989.
\newblock Adl: Exploringthe middle ground between strips and the situation calculus.
\newblock In \emph{Proceedings of the First International Conference on Principles of Knowledge Representation and Reasoning (KR’89)}, 324--332.

\bibitem[{Scala et~al.(2016)Scala, Haslum, Thi{\'e}baux, and Ramirez}]{scala2016interval}
Scala, E.; Haslum, P.; Thi{\'e}baux, S.; and Ramirez, M. 2016.
\newblock Interval-based relaxation for general numeric planning.
\newblock In \emph{European Conference on Artificial Intelligence (ECAI)}, 655--663.

\bibitem[{Seipp, Torralba, and Hoffmann(2022)}]{seipp-et-al-zenodo2022}
Seipp, J.; Torralba, {\'A}.; and Hoffmann, J. 2022.
\newblock {PDDL} Generators.
\newblock \url{https://doi.org/10.5281/zenodo.6382173}.

\bibitem[{Stern and Juba(2017)}]{stern2017efficientAndSafe}
Stern, R.; and Juba, B. 2017.
\newblock Efficient, Safe, and Probably Approximately Complete Learning of Action Models.
\newblock In \emph{International Joint Conference on Artificial Intelligence ({IJCAI})}, 4405--4411.

\bibitem[{Vallati et~al.(2015)Vallati, Chrpa, Grze{\'s}, McCluskey, Roberts, Sanner et~al.}]{vallati20152014}
Vallati, M.; Chrpa, L.; Grze{\'s}, M.; McCluskey, T.~L.; Roberts, M.; Sanner, S.; et~al. 2015.
\newblock The 2014 international planning competition: Progress and trends.
\newblock \emph{AI Magazine}, 36(3): 90--98.

\bibitem[{Wachi and Sui(2020)}]{wachi2020safe}
Wachi, A.; and Sui, Y. 2020.
\newblock Safe Reinforcement Learning in Constrained {M}arkov Decision Processes.
\newblock In \emph{International Conference on Machine Learning (ICML)}, 9797--9806.

\bibitem[{Yang, Wu, and Jiang(2007)}]{yang2007learning}
Yang, Q.; Wu, K.; and Jiang, Y. 2007.
\newblock Learning action models from plan examples using weighted MAX-SAT.
\newblock \emph{Artificial Intelligence}, 171(2-3): 107--143.

\bibitem[{Zhuo et~al.(2010)Zhuo, Yang, Hu, and Li}]{zhuo2010learning}
Zhuo, H.~H.; Yang, Q.; Hu, D.~H.; and Li, L. 2010.
\newblock Learning complex action models with quantifiers and logical implications.
\newblock \emph{Artificial Intelligence}, 174(18): 1540--1569.

\end{thebibliography}


\end{document}